\documentclass[pdflatex,sn-mathphys-num]{sn-jnl}

\usepackage{graphicx}
\usepackage{multirow}
\usepackage{amsmath,amssymb,amsfonts}
\usepackage{mathrsfs}
\usepackage[title]{appendix}
\usepackage{xcolor}
\usepackage{textcomp}
\usepackage{manyfoot}
\usepackage{booktabs}
\usepackage{algorithm}
\usepackage{algorithmicx}
\usepackage{algpseudocode}
\usepackage{listings}
\usepackage{url}
\usepackage{caption}

\newcommand{\rbi}{\mathbf{i}}
\newcommand{\rbj}{\mathbf{j}}
\newcommand{\rbk}{\mathbf{k}}
\algrenewcommand\algorithmicrequire{\textbf{Input:}}
\algrenewcommand\algorithmicensure{\textbf{Output:}}

\theoremstyle{thmstyleone}
\newtheorem{theorem}{Theorem}[section]

\newtheorem{corollary}[theorem]{Corollary}
\newtheorem{lemma}[theorem]{Lemma}

\newtheorem{proposition}[theorem]{Proposition}

\theoremstyle{thmstyletwo}
\newtheorem{remark}[theorem]{Remark}

\theoremstyle{thmstylethree}
\newtheorem{definition}[theorem]{Definition}

\begin{document}
	
	\title{Graph Regularized Non-negative Reduced Biquaternion Matrix Factorization for Color Image Recognition}
	
	\author[1]{\fnm{Hailang} \sur{Wu}}
	\author[1]{\fnm{Yonghe} \sur{Liu}}
	\author[1]{\fnm{Bingxuan} \sur{Yu}}
	\author[1]{\fnm{Chaoqian} \sur{Li}}\email{wuhailang0217@163.com,~13625698503@163.com,\\2284047786@qq.com,~lichaoqian@ynu.edu.cn}
	
	\affil[1]{\orgdiv{School of Mathematics and Statistics}, \orgname{Yunnan University}, \orgaddress{\city{Kunming}, \country{China}}}
	
	\abstract{Non-negative reduced biquaternion matrix factorization (NRBMF) uses the product of reduced biquaternion (RB) matrices to incorporate the non-negativity constraints of color image pixels into the factorization process. However, NRBMF mainly focuses on reconstruction accuracy and does not explicitly exploit the local geometric structure of image data, which may limit the discriminative ability of the obtained low-dimensional coefficient representations. To address this issue, we propose a graph regularized non-negative reduced biquaternion matrix factorization (GNRBMF) model for color image recognition. The proposed model incorporates a graph Laplacian regularizer into the reduced biquaternion coefficient matrix, encouraging nearby samples in the original space to have similar coefficient representations. Meanwhile, GNRBMF retains the non-negativity property of NRBMF in the reduced biquaternion algebra. To solve the optimization problem, a component-wise alternating projected gradient algorithm is derived, and its convergence properties are analyzed. Experimental results on three color image datasets show that the proposed GNRBMF model achieves competitive or superior recognition performance compared with several methods in most tested settings.}
	
	\maketitle
	
	\section{Introduction}
	
	Color image recognition is an important task in computer vision and pattern recognition, and has been widely used in applications such as face recognition \cite{lu2018color}, medical imaging \cite{ashikuzzaman2020low}, and remote sensing \cite{kior2024rgb}. Compared with grayscale images, color images provide richer visual information because the RGB channels describe complementary aspects of object appearance. Therefore, an effective color image representation should extract compact features while preserving useful correlations among RGB channels \cite{cui2020color}. In addition, the pixel values of color images are non-negative, which makes non-negative representation models naturally suitable for color image feature extraction.
	
	Non-negative matrix factorization (NMF) \cite{lee1999learning} is a classical non-negative representation method for data representation and feature extraction \cite{lee2000algorithms}. Given a non-negative data matrix, NMF factorizes it into a non-negative basis matrix and a non-negative coefficient matrix. In image recognition, the basis matrix can capture latent local patterns, while the coefficient matrix provides low-dimensional features for classification \cite{guillamet2002non,wang2005non,chen2022novel}. Owing to the non-negativity constraint, NMF often yields parts-based and interpretable representations. However, standard NMF is defined in the real-valued matrix setting and cannot naturally represent the RGB channels within a unified algebraic framework. When it is applied to color images, the RGB channels are usually processed independently or simply concatenated into a long real-valued vector, which may weaken the intrinsic correlations among different color channels \cite{xu2015vector}.
	
	Quaternion algebra \cite{hamilton1866elements} offers a useful hypercomplex representation for color image data. By encoding the three RGB channels into the imaginary components of a quaternion, each color pixel can be represented as a single algebraic entity, which helps preserve and exploit the intrinsic relationships among RGB channels. Quaternion-based methods have been applied to color image recognition \cite{zou2016quaternion}, denoising \cite{yu2019quaternion}, restoration \cite{chen2019low}, and inpainting \cite{miao2020quaternion}. To incorporate non-negativity into quaternion-based factorization, Ke et al. \cite{ke2023quasi} proposed quasi non-negative quaternion matrix factorization (QNQMF). However, due to the noncommutativity of quaternion multiplication, the QNQMF model cannot theoretically guarantee that the product of two quasi non-negative quaternion factor matrices remains quasi non-negative.
	
	To overcome this limitation, Miao et al. \cite{miao2024non} proposed non-negative reduced biquaternion matrix factorization (NRBMF) for color image recognition. Reduced biquaternion algebra has a four-component structure suitable for color image representation and, unlike quaternion algebra, has commutative multiplication \cite{pei2004commutative}. In NRBMF, the coefficient matrix is restricted to a special reduced biquaternion form, which guarantees that the product of the factor matrices remains non-negative in the reduced biquaternion algebra. Thus, NRBMF provides a more rigorous framework for hypercomplex non-negative factorization than QNQMF.
	
	Despite these advantages, NRBMF mainly focuses on reconstruction accuracy and does not explicitly use the local geometric structure of image data. In many recognition tasks, visually similar samples are expected to have similar low-dimensional coefficient representations. Ignoring this structure may limit the discriminative ability of the obtained features, especially when the images contain illumination changes, pose variations, background interference, or large within-class differences. Graph regularization is an effective strategy for preserving local geometric structure in representation learning \cite{cai2011graph}. Its basic principle is that nearby samples in the original space should remain close in the obtained low-dimensional feature space. Since the coefficient matrix directly provides the low-dimensional representation of each sample, it is natural to impose the graph constraint on the coefficient matrix.
	
	To address the limitation that NRBMF does not explicitly exploit the local geometric structure of samples, this paper proposes a graph regularized non-negative reduced biquaternion matrix factorization (GNRBMF) model for color image recognition. The proposed model extends NRBMF by introducing a graph Laplacian regularizer into the reduced biquaternion coefficient matrix. In this way, the obtained coefficient representation preserves local geometric structure while retaining the non-negativity structure of NRBMF.
	
	The main contributions of this paper are summarized as follows.
	
	\begin{itemize}
		\item We propose a graph regularized non-negative reduced biquaternion matrix factorization model for color image recognition. The proposed model incorporates local geometric information into the reduced biquaternion coefficient representation, while preserving the non-negativity structure and the ability of RB representation to model RGB channel correlations. To solve the resulting constrained optimization problem, we further derive a component-wise alternating projected gradient algorithm.
		
		\item We provide a convergence analysis for the proposed optimization algorithm. The objective values generated by the algorithm are proved to be nonincreasing and convergent. Moreover, if the generated sequence converges, then its limit point satisfies the stationarity conditions of the proposed constrained optimization problem.
		
		\item We conduct extensive experiments on CASIA-FaceV5, KDEF, and Asirra for color image recognition. The results show that the proposed GNRBMF model achieves competitive or superior recognition performance compared with several real-valued, quaternion-based, and reduced-biquaternion-based methods in most tested settings.
	\end{itemize}
	
	The rest of this paper is organized as follows. Reduced biquaternion algebra and NRBMF are introduced in Section 2. The proposed GNRBMF model is presented in Section 3. The optimization algorithm is developed and the convergence analysis is given in Section 4. The experimental results are reported in Section 5. Finally, concluding remarks are provided in Section 6.

	\section{Preliminaries}
	
	In this section, we introduce the preliminaries used in this paper. Some basic notations used in the paper are given in Table~\ref{tab:notations}.
	
	\begin{table}[htbp]
		\caption{Basic notations used in this paper}\label{tab:notations}
		\begin{tabular}{@{}ll@{}}
			\toprule
			Notation & Representation \\
			\midrule
			\(\mathbb R\) & Real space \\
			\(\mathbb{RB}\) & Reduced biquaternion space \\
			\(q\), \(\mathbf q\), \(Q\) & Real scalar, vector, and matrix \\
			\(\ddot q\), \(\ddot{\mathbf q}\), \(\ddot Q\) & RB scalar, vector, and matrix \\
			\(\rbi,\rbj,\rbk\) & Imaginary units of reduced biquaternions \\
			\(\operatorname{Re}(\cdot)\) & Real component of an RB number or matrix \\
			\(\operatorname{Im}_{\rbi}(\cdot)\), \(\operatorname{Im}_{\rbj}(\cdot)\), \(\operatorname{Im}_{\rbk}(\cdot)\) & Three imaginary components of an RB number or matrix \\
			\((\cdot)^\top\), \((\cdot)^*\), \((\cdot)^H\) & Transpose, conjugate, and conjugate transpose \\
			\(\langle\cdot,\cdot\rangle\) & Inner product of RB matrices \\
			\(|\cdot|\), \(\|\cdot\|_F\), \(\|\cdot\|_2\) & Modulus, Frobenius norm, and spectral norm of a matrix or Euclidean norm of a vector \\
			\(\operatorname{tr}(\cdot)\) & Trace of a square matrix \\
			\(p\) & Number of nearest neighbors in graph construction \\
			\(\lambda\) & Graph regularization parameter \\
			\(\otimes\) & Hadamard product \\
			\(\triangleq\) & Defined as \\
			\botrule
		\end{tabular}
	\end{table}
	
	\subsection{Reduced Biquaternion}
	
	A reduced biquaternion (RB) number \(\ddot q\in \mathbb{RB}\) is defined as
	\cite{pei2004commutative,schutte1990hypercomplex,dimitrov1992multiplication}
	\[
	\ddot q=q_0+q_1\rbi+q_2\rbj+q_3\rbk,
	\]
	where \(q_\ell\in\mathbb{R}\) for \(\ell=0,1,2,3\). The imaginary units \(\rbi,\rbj,\rbk\) satisfy
	\[
	\rbi^2=\rbk^2=-1,\quad \rbj^2=1,
	\qquad
	\rbi\rbj=\rbj\rbi=\rbk,\quad \rbj\rbk=\rbk\rbj=\rbi,\quad \rbk\rbi=\rbi\rbk=-\rbj.
	\]
	Unlike standard quaternion multiplication, reduced biquaternion multiplication is commutative.
	
	For \(\ddot q=q_0+q_1\rbi+q_2\rbj+q_3\rbk\in \mathbb{RB}\), its conjugate and modulus are defined as
	\cite{pei2008eigenvalues,elmelegy2022linear}
	\[
	\ddot q^{*}=q_0-q_1\rbi+q_2\rbj-q_3\rbk,
	\]
	and
	\[
	|\ddot q|=
	\sqrt{q_0^2+q_1^2+q_2^2+q_3^2}.
	\]
	
	Let \(\ddot Q=(\ddot q_{mn})\in \mathbb{RB}^{M\times N}\) be an RB matrix. It can be written as
	\[
	\ddot Q=Q_0+Q_1\rbi+Q_2\rbj+Q_3\rbk,
	\]
	where \(Q_\ell\in\mathbb{R}^{M\times N}\) for \(\ell=0,1,2,3\). The transpose, conjugate, and conjugate transpose of \(\ddot Q\) are denoted by \(\ddot Q^\top\), \(\ddot Q^*\), and \(\ddot Q^H\), respectively. If \(Q_0=0\), then \(\ddot Q\) is called a pure RB matrix.
	
	For two RB matrices \(\ddot Q=(\ddot q_{mn}),\ddot P=(\ddot p_{mn})\in \mathbb{RB}^{M\times N}\), their inner product is defined as
	\[
	\langle \ddot Q,\ddot P\rangle
	=
	\sum_{m=1}^{M}\sum_{n=1}^{N}\ddot q_{mn}^{*}\ddot p_{mn}.
	\]
	The Frobenius norm of \(\ddot Q\) is defined as
	\[
	\|\ddot Q\|_F
	=
	\sqrt{\operatorname{Re}\big(\langle \ddot Q,\ddot Q\rangle\big)}
	=
	\left(
	\sum_{m=1}^{M}\sum_{n=1}^{N}
	|\ddot q_{mn}|^2
	\right)^{1/2}.
	\]
	
	\begin{definition} \cite{miao2024non}[Non-negative RB matrix]
		An RB matrix
		\[
		\ddot Q=Q_0+Q_1\rbi+Q_2\rbj+Q_3\rbk
		\]
		is called non-negative if all its component matrices are real non-negative matrices, that is,
		\[
		Q_\ell\ge 0,\qquad \ell=0,1,2,3.
		\]
		The set of all non-negative RB matrices of size \(M\times N\) is denoted by \(\mathbb{RB}_{+}^{M\times N}\). In particular, \(\mathbb{RB}_{+\rbj}^{M\times N}\) denotes the subset of non-negative RB matrices whose first and third imaginary component matrices are zero. Hence, if \(\ddot Q\in \mathbb{RB}_{+\rbj}^{M\times N}\), then
		\begin{equation}
			\ddot Q=Q_0+Q_2\rbj,
			\qquad Q_0\ge 0,
			\quad Q_2\ge 0.
			\label{eq:rb_plus_j}
		\end{equation}
	\end{definition}
	
	\subsection{Non-negative Reduced Biquaternion Matrix Factorization}
	\label{subsec:nrbmf}
	
	Given a non-negative RB matrix \(\ddot X=X_0+X_1\rbi+X_2\rbj+X_3\rbk\in \mathbb{RB}_{+}^{M\times N}\), NRBMF seeks two factor matrices \(\ddot W=W_0+W_1\rbi+W_2\rbj+W_3\rbk\in \mathbb{RB}_{+}^{M\times l}\) and \(\ddot H=H_0+H_2\rbj\in \mathbb{RB}_{+\rbj}^{l\times N}\), such that
	\[
	\ddot X\approx \ddot W\ddot H,
	\]
	where \(l\leq \min\{M,N\}\) is the factorization rank. Here, \(\ddot W\) is the basis matrix and \(\ddot H\) is the coefficient matrix. The constraint \(\ddot H\in \mathbb{RB}_{+\rbj}^{l\times N}\) is essential because it ensures that the product \(\ddot W\ddot H\) remains non-negative in the RB algebra. Indeed, using the multiplication rules of reduced biquaternions, we have
	\begin{equation}
		\label{eq:rb_product_nrbmf}
		\begin{aligned}
			\ddot W\ddot H
			&=
			(W_0H_0+W_2H_2)
			+
			(W_1H_0+W_3H_2)\rbi  \\
			&\quad+
			(W_0H_2+W_2H_0)\rbj
			+
			(W_1H_2+W_3H_0)\rbk.
		\end{aligned}
	\end{equation}
	If \(W_0,W_1,W_2,W_3,H_0\), and \(H_2\) are elementwise non-negative, then each component of \(\ddot W\ddot H\) is also elementwise non-negative. Therefore, the structured form \(\ddot H=H_0+H_2\rbj\) guarantees that \(\ddot W\ddot H\) remains non-negative.
	
	The NRBMF model \cite{miao2024non} is formulated as
	\begin{equation}
		\begin{aligned}
			\min_{\ddot W,\ddot H}\quad
			& \frac{1}{2}\|\ddot X-\ddot W\ddot H\|_F^2 \\
			\text{s.t.}\quad
			& \ddot W\in \mathbb{RB}_{+}^{M\times l},\quad
			\ddot H\in \mathbb{RB}_{+\rbj}^{l\times N}.
		\end{aligned}
		\label{eq:nrbmf_model}
	\end{equation}
	
	In this model, \(\ddot W\) can be regarded as the RB basis matrix, while \(\ddot H\) provides a low-dimensional coefficient representation of the samples. For color image data, the RB representation can encode different color-related components in a unified algebraic framework. Meanwhile, the non-negativity constraints preserve the physical meaning of image intensity values and lead to interpretable factorization results.
	
	\section{The Proposed GNRBMF Model}
	\label{sec:gnrbmf_model}
	
	For color image recognition, the coefficient matrix obtained by NRBMF is used as the low-dimensional feature representation. However, NRBMF mainly focuses on the reconstruction error and does not explicitly exploit the local geometric structure among training samples. In many recognition tasks, nearby samples in the original data space are expected to have similar low-dimensional coefficient representations. To incorporate this structural information while maintaining the non-negativity structure of NRBMF, we introduce a graph Laplacian regularization term into the RB coefficient matrix.
	
	Let \(\ddot X=[\ddot{\mathbf x}_1,\ddot{\mathbf x}_2,\ldots,\ddot{\mathbf x}_N]
	\in \mathbb{RB}_{+}^{M\times N}\) be the non-negative RB data matrix, where each column \(\ddot{\mathbf x}_n\) denotes one vectorized color image sample. We construct an undirected \(p\)-nearest-neighbor graph on the training samples. Each sample is treated as a vertex, and the edge weights are stored in the affinity matrix \(A=[A_{mn}]\in\mathbb{R}^{N\times N}\).
	
	The squared distance between two samples \(\ddot{\mathbf x}_m\) and \(\ddot{\mathbf x}_n\) is defined as
	\[
	d_{mn}^2=\|\ddot{\mathbf x}_m-\ddot{\mathbf x}_n\|_F^2.
	\]
	For each sample \(\ddot{\mathbf x}_n\), let \(\mathcal N_p(n)\) denote the set of its \(p\) nearest neighbors, excluding \(\ddot{\mathbf x}_n\) itself. The affinity matrix \(A\) is defined by the following \(0\)--\(1\) weighting rule:
	\begin{equation}
		A_{mn}
		=
		\left\{
		\begin{array}{ll}
			1, & m\in \mathcal N_p(n)\ \text{or}\ n\in \mathcal N_p(m),\\
			0, & \text{otherwise}.
		\end{array}
		\right.
		\label{eq:affinity_general}
	\end{equation}
	The use of the ``or'' condition makes the graph undirected, and hence \(A\) is symmetric. The degree matrix \(D\) and the graph Laplacian \(L\) are defined by
	\[
	D_{mm}=\sum_{n=1}^{N}A_{mn},
	\qquad
	L=D-A.
	\]
	The Laplacian matrix \(L\) is then used to regularize the coefficient representation, so that nearby samples in the original RB sample space tend to remain close in the low-dimensional feature space.
	
	For a real-valued coefficient matrix
	\[
	S=[\mathbf s_1,\mathbf s_2,\ldots,\mathbf s_N]\in\mathbb R^{l\times N},
	\]
	where \(\mathbf s_n\) denotes the low-dimensional representation of the \(n\)-th sample, the standard graph identity gives
	\begin{equation}
		\frac{1}{2}
		\sum_{m,n=1}^{N}
		A_{mn}\|\mathbf s_m-\mathbf s_n\|_2^2
		=
		\operatorname{tr}(SLS^\top),
		\label{eq:graph_identity}
	\end{equation}
	provided that \(A\) is symmetric. Therefore, minimizing \(\operatorname{tr}(SLS^\top)\) encourages adjacent samples on the graph to have similar coefficient vectors.
	
	We now apply this idea to the structured RB coefficient matrix in NRBMF\@. In the NRBMF framework, the coefficient matrix has the form
	\[
	\ddot H=H_0+H_2\rbj\in \mathbb{RB}_{+\rbj}^{l\times N}.
	\]
	The coefficient vector of the \(n\)-th sample can thus be written as
	\[
	\ddot{\mathbf h}_n=\mathbf h_{0,n}+\mathbf h_{2,n}\rbj,
	\]
	where \(\mathbf h_{0,n}\) and \(\mathbf h_{2,n}\) are the \(n\)-th columns of \(H_0\) and \(H_2\), respectively. Since the columns of \(\ddot H\) serve as the low-dimensional representations of the samples, we define the graph regularization term in the RB coefficient space as
	\[
	R(\ddot H)
	=
	\frac{1}{2}
	\sum_{m,n=1}^{N}
	A_{mn}\|\ddot{\mathbf h}_m-\ddot{\mathbf h}_n\|_F^2.
	\]
	For any two coefficient vectors \(\ddot{\mathbf h}_m\) and \(\ddot{\mathbf h}_n\), we have
	\[
	\ddot{\mathbf h}_m-\ddot{\mathbf h}_n
	=
	(\mathbf h_{0,m}-\mathbf h_{0,n})+(\mathbf h_{2,m}-\mathbf h_{2,n})\rbj.
	\]
	Hence,
	\[
	\|\ddot{\mathbf h}_m-\ddot{\mathbf h}_n\|_F^2
	=
	\|\mathbf h_{0,m}-\mathbf h_{0,n}\|_2^2
	+
	\|\mathbf h_{2,m}-\mathbf h_{2,n}\|_2^2.
	\]
	Applying the graph identity in~\eqref{eq:graph_identity} to \(H_0\) and \(H_2\), respectively, yields
	\begin{equation}
		R(\ddot H)
		=
		\operatorname{tr}(H_0LH_0^\top)
		+
		\operatorname{tr}(H_2LH_2^\top).
		\label{eq:rb_graph_regularizer}
	\end{equation}
	Thus, the graph regularization term is imposed on the two nonzero components of \(\ddot H\), namely \(H_0\) and \(H_2\).
	
	By combining the reconstruction term with the graph regularization term in~\eqref{eq:rb_graph_regularizer}, we obtain the proposed graph regularized non-negative reduced biquaternion matrix factorization model:
	\begin{equation}
		\begin{aligned}
			\min_{\ddot W,\ddot H}\quad
			& \frac{1}{2}\|\ddot X-\ddot W\ddot H\|_F^2
			+\frac{\lambda}{2}
			\left(
			\operatorname{tr}(H_0LH_0^\top)
			+
			\operatorname{tr}(H_2LH_2^\top)
			\right) \\
			\text{s.t.}\quad
			& \ddot W\in \mathbb{RB}_{+}^{M\times l},\quad
			\ddot H\in \mathbb{RB}_{+\rbj}^{l\times N},
		\end{aligned}
		\label{eq:gnrbmf_model}
	\end{equation}
	where \(\lambda\ge 0\) is the regularization parameter that balances reconstruction accuracy and graph smoothness. When \(\lambda=0\), the proposed model reduces to NRBMF\@. Thus, GNRBMF extends NRBMF by incorporating local geometric information into the coefficient representation while maintaining the non-negativity structure of NRBMF.
	
	\section{The Optimization Algorithm}
	\label{sec:optimization}
	
	In this section, an alternating projected gradient algorithm is developed for solving the proposed GNRBMF model. The basic notations used below follow Table~\ref{tab:notations}.
	
	\subsection{Alternating Projected Gradient Algorithm}
	\label{subsec:apg_algorithm}
	
	To solve the GNRBMF model in \eqref{eq:gnrbmf_model}, we denote its objective function by \(F(\ddot W,\ddot H)\).
	
	Using the RB product formula in~\eqref{eq:rb_product_nrbmf}, we define the component-wise residual matrices by
	\begin{equation}
		\label{eq:residual_components}
		\begin{aligned}
			E_0&=W_0H_0+W_2H_2-X_0, &
			E_1&=W_1H_0+W_3H_2-X_1,\\
			E_2&=W_0H_2+W_2H_0-X_2, &
			E_3&=W_1H_2+W_3H_0-X_3.
		\end{aligned}
	\end{equation}
	Then
	\[
	\ddot W\ddot H-\ddot X=E_0+E_1\rbi+E_2\rbj+E_3\rbk.
	\]
	Hence the objective function \(F(\ddot W,\ddot H)\) can be rewritten as
	\begin{equation}
		\label{eq:objective_component_form}
		F(\ddot W,\ddot H)
		=
		\frac{1}{2}\sum_{r=0}^{3}\|E_r\|_F^2
		+\frac{\lambda}{2}
		\left[
		\operatorname{tr}(H_0LH_0^\top)
		+\operatorname{tr}(H_2LH_2^\top)
		\right].
	\end{equation}
	
	Based on the component representation in~\eqref{eq:objective_component_form}, we next derive the gradients used in the alternating projected gradient algorithm. Since \(F\) is a real-valued function of RB matrix variables, its gradient with respect to \(\ddot W\) is defined component-wise as
	\[
	\nabla_{\ddot W}F
	\triangleq
	\frac{\partial F}{\partial W_0}
	+\frac{\partial F}{\partial W_1}\rbi
	+\frac{\partial F}{\partial W_2}\rbj
	+\frac{\partial F}{\partial W_3}\rbk.
	\]
	By standard matrix differentiation,
	\begin{equation}
		\label{eq:grad_W_components}
		\begin{aligned}
			\frac{\partial F}{\partial W_0}
			&=E_0H_0^\top+E_2H_2^\top,
			&
			\frac{\partial F}{\partial W_1}
			&=E_1H_0^\top+E_3H_2^\top,\\
			\frac{\partial F}{\partial W_2}
			&=E_0H_2^\top+E_2H_0^\top,
			&
			\frac{\partial F}{\partial W_3}
			&=E_1H_2^\top+E_3H_0^\top.
		\end{aligned}
	\end{equation}
	For compact notation, this component-wise gradient can be written as
	\begin{equation}
		\label{eq:grad_W_compact}
		\nabla_{\ddot W}F(\ddot W,\ddot H)
		=
		(\ddot W\ddot H-\ddot X)\ddot H^{H}.
	\end{equation}
	The compact expression in~\eqref{eq:grad_W_compact} is only a shorthand notation for the corresponding real component-wise gradients.
	
	For the matrix \(\ddot H=H_0+H_2\rbj\), only the real component and the \(\rbj\)-component are involved. Thus,
	\[
	\nabla_{\ddot H}F
	\triangleq
	\frac{\partial F}{\partial H_0}
	+\frac{\partial F}{\partial H_2}\rbj.
	\]
	Since the affinity matrix \(A\) is symmetric, the graph Laplacian \(L=D-A\) is also symmetric. Therefore, for \(r=0,2\),
	\[
	\frac{\partial}{\partial H_r}
	\left[
	\frac{\lambda}{2}\operatorname{tr}(H_rLH_r^\top)
	\right]
	=\lambda H_rL.
	\]
	Combining the reconstruction term and the graph regularization term, we obtain
	\begin{equation}
		\label{eq:grad_H_components}
		\begin{aligned}
			\frac{\partial F}{\partial H_0}
			&=
			W_0^\top E_0
			+W_1^\top E_1
			+W_2^\top E_2
			+W_3^\top E_3
			+\lambda H_0L,\\
			\frac{\partial F}{\partial H_2}
			&=
			W_2^\top E_0
			+W_3^\top E_1
			+W_0^\top E_2
			+W_1^\top E_3
			+\lambda H_2L.
		\end{aligned}
	\end{equation}
	The gradients in~\eqref{eq:grad_W_components} and~\eqref{eq:grad_H_components} are then used to construct the alternating projected gradient updates.
	
	The constrained GNRBMF problem is solved by alternately updating \(\ddot W\) and \(\ddot H\). Before presenting the update rules, we first define the projections used in the algorithm.
	
	Let \(P_C(\cdot)\) denote the Euclidean projection onto a closed convex set \(C\). In particular, let \(P_{\mathbb R_{+}^{m\times n}}\) denote the element-wise projection onto the non-negative real matrix cone, namely \(P_{\mathbb R_{+}^{m\times n}}(Y)\triangleq \max(Y,0)\). For \(\ddot Q=Q_0+Q_1\rbi+Q_2\rbj+Q_3\rbk\in \mathbb{RB}^{m\times n}\), the projection onto \(\mathbb{RB}_{+}^{m\times n}\) is
	\begin{equation}
		\label{eq:projection_RB_plus}
		\begin{aligned}
			P_{\mathbb{RB}_{+}^{m\times n}}(\ddot Q)
			\triangleq&
			P_{\mathbb R_{+}^{m\times n}}(Q_0)
			+P_{\mathbb R_{+}^{m\times n}}(Q_1)\rbi
			+P_{\mathbb R_{+}^{m\times n}}(Q_2)\rbj
			+P_{\mathbb R_{+}^{m\times n}}(Q_3)\rbk.
		\end{aligned}
	\end{equation}
	The projection onto the structured feasible set \(\mathbb{RB}_{+\rbj}^{m\times n}\) is
	\begin{equation}
		\label{eq:projection_RB_plus_j}
		P_{\mathbb{RB}_{+\rbj}^{m\times n}}(\ddot Q)
		\triangleq
		P_{\mathbb R_{+}^{m\times n}}(Q_0)
		+P_{\mathbb R_{+}^{m\times n}}(Q_2)\rbj.
	\end{equation}
	Thus, this projection keeps and non-negativizes the real and \(\rbj\)-components while removing the \(\rbi\)- and \(\rbk\)-components.
	
	Given the current iterate \((\ddot W^t,\ddot H^t)\), the projected gradient updates are
	\begin{equation}
		\label{eq:W_update}
		\ddot W^{t+1}
		=
		P_{\mathbb{RB}_{+}^{M\times l}}
		\left[
		\ddot W^t
		-\alpha_t\nabla_{\ddot W}F(\ddot W^t,\ddot H^t)
		\right],
	\end{equation}
	and
	\begin{equation}
		\label{eq:H_update}
		\ddot H^{t+1}
		=
		P_{\mathbb{RB}_{+\rbj}^{l\times N}}
		\left[
		\ddot H^t
		-\beta_t\nabla_{\ddot H}F(\ddot W^{t+1},\ddot H^t)
		\right].
	\end{equation}
	The update of \(\ddot H\) uses the updated value \(\ddot W^{t+1}\), yielding a Gauss--Seidel type alternating scheme.
	
	To ensure sufficient descent, the step sizes \(\alpha_t\) and \(\beta_t\) are selected by Armijo backtracking line search. Let \(0<\mu<1\) and \(0<\sigma<1\). For the \(\ddot W\)-update, set \(\alpha_t=\mu^{d_t}\), where \(d_t\) is the first non-negative integer such that
	\begin{equation}
		\label{eq:armijo_W}
		\begin{aligned}
			F(\ddot W^{t+1},\ddot H^t)-F(\ddot W^t,\ddot H^t)
			&\le
			\sigma
			\operatorname{Re}\!\left(\left\langle
			\nabla_{\ddot W}F(\ddot W^t,\ddot H^t),
			\ddot W^{t+1}-\ddot W^t
			\right\rangle\right).
		\end{aligned}
	\end{equation}
	For the \(\ddot H\)-update, set \(\beta_t=\mu^{s_t}\), where \(s_t\) is the first non-negative integer such that
	\begin{equation}
		\label{eq:armijo_H}
		\begin{aligned}
			F(\ddot W^{t+1},\ddot H^{t+1})
			-F(\ddot W^{t+1},\ddot H^t)
			&\le
			\sigma
			\operatorname{Re}\!\left(\left\langle
			\nabla_{\ddot H}F(\ddot W^{t+1},\ddot H^t),
			\ddot H^{t+1}-\ddot H^t
			\right\rangle\right).
		\end{aligned}
	\end{equation}
	The real parts of the RB inner products in~\eqref{eq:armijo_W} and~\eqref{eq:armijo_H} are evaluated through their corresponding real component representations.
	
	The complete procedure is summarized in Algorithm~\ref{alg:gnrbmf}.
	
	\begin{algorithm}[htbp]
		\caption{RB projected gradient method for GNRBMF}\label{alg:gnrbmf}
		\begin{algorithmic}[1]
			\Require \(\ddot X\in\mathbb{RB}_{+}^{M\times N}\); factorization rank \(l\); graph Laplacian \(L\); regularization parameter \(\lambda\); parameters \(0<\mu<1\), \(0<\sigma<1\); tolerance \(\mathrm{tol}>0\); maximum number of iterations \(T_{\max}\).
			\Ensure \(\ddot W\in\mathbb{RB}_{+}^{M\times l}\) and \(\ddot H\in\mathbb{RB}_{+\rbj}^{l\times N}\).
			\State Initialize \(\ddot W^0\in\mathbb{RB}_{+}^{M\times l}\) and \(\ddot H^0\in\mathbb{RB}_{+\rbj}^{l\times N}\); set \(t=0\) and \(r=+\infty\).
			\While{\(t<T_{\max}\) and \(r\ge \mathrm{tol}\)}
			\State Compute \(G_W^t=\nabla_{\ddot W}F(\ddot W^t,\ddot H^t)\) by~\eqref{eq:grad_W_components}.
			\State Update \(\ddot W^{t+1}=P_{\mathbb{RB}_{+}^{M\times l}}(\ddot W^t-\alpha_tG_W^t)\), where \(\alpha_t=\mu^{d_t}\), and \(d_t\) is the first non-negative integer for which~\eqref{eq:armijo_W} holds.
			\State Compute \(G_H^t=\nabla_{\ddot H}F(\ddot W^{t+1},\ddot H^t)\) by~\eqref{eq:grad_H_components}.
			\State Update \(\ddot H^{t+1}=P_{\mathbb{RB}_{+\rbj}^{l\times N}}(\ddot H^t-\beta_tG_H^t)\), where \(\beta_t=\mu^{s_t}\), and \(s_t\) is the first non-negative integer for which~\eqref{eq:armijo_H} holds.
			\State Compute \(\displaystyle r=\frac{\|\ddot W^{t+1}\ddot H^{t+1}-\ddot W^t\ddot H^t\|_F}{\|\ddot W^t\ddot H^t\|_F}\).
			\State Set \(t\leftarrow t+1\).
			\EndWhile
			\State \Return \(\ddot W^t\) and \(\ddot H^t\).
		\end{algorithmic}
	\end{algorithm}
	
	\begin{remark}
		The computationally expensive step in Algorithm~\ref{alg:gnrbmf} is the search for the step sizes \(\alpha_t\) and \(\beta_t\). In numerical implementation, the backtracking procedure may be initialized by the previously accepted step sizes and then adjusted until~\eqref{eq:armijo_W} and~\eqref{eq:armijo_H} are satisfied. This strategy is commonly used in projected gradient algorithms for matrix factorization and can reduce the number of function evaluations. For the convergence analysis below, the Armijo backtracking procedures are understood to start from the fixed trial step size \(1\) at each iteration.
	\end{remark}
	
	\subsection{Convergence Analysis}
	
	We now analyze the convergence of Algorithm~\ref{alg:gnrbmf}. Since the objective function is real-valued and the variables are RB matrices, the analysis is carried out through their real component representations. For an arbitrary RB matrix \(\ddot U=U_0+U_1\rbi+U_2\rbj+U_3\rbk\in \mathbb{RB}^{m\times n}\), define
	\[
	\Phi(\ddot U)
	\triangleq
	\begin{bmatrix}
		\operatorname{vec}(U_0)\\
		\operatorname{vec}(U_1)\\
		\operatorname{vec}(U_2)\\
		\operatorname{vec}(U_3)
	\end{bmatrix}
	\in \mathbb R^{4mn}.
	\]
	where \(\operatorname{vec}(\cdot)\) denotes the vectorization operation of a matrix. For another RB matrix \(\ddot V=V_0+V_1\rbi+V_2\rbj+V_3\rbk\in \mathbb{RB}^{m\times n}\), we have
	\[
	\operatorname{Re}\!\left(\langle \ddot U,\ddot V\rangle\right)
	=
	\Phi(\ddot U)^\top \Phi(\ddot V),
	\qquad
	\|\ddot U\|_F=\|\Phi(\ddot U)\|_2.
	\]
	For the structured matrix \(\ddot H=H_0+H_2\rbj\), define
	\[
	\Phi_{\rbj}(\ddot H)
	\triangleq
	\begin{bmatrix}
		\operatorname{vec}(H_0)\\
		\operatorname{vec}(H_2)
	\end{bmatrix}.
	\]
	Under these identifications, \(\mathbb{RB}_{+}^{M\times l}\) and \(\mathbb{RB}_{+\rbj}^{l\times N}\) are closed convex cones in finite-dimensional Euclidean spaces. All projection and convergence arguments below are understood under this equivalent real component representation. In particular, a stationary point refers to a first-order stationary point under this real component representation.
	
	We first state the first-order stationarity conditions for the problem~\eqref{eq:gnrbmf_model}.
	
	\begin{proposition}
		\label{prop:kkt}
		A feasible point \((\ddot W^\#, \ddot H^\#)\) is a stationary point of the problem~\eqref{eq:gnrbmf_model} if and only if
		\begin{equation}
			\label{eq:kkt_component}
			\left\{
			\begin{aligned}
				&\ddot W^\#\in \mathbb{RB}_{+}^{M\times l},\qquad
				\ddot H^\#\in \mathbb{RB}_{+\rbj}^{l\times N},\\
				&\nabla_{\ddot W}F(\ddot W^\#,\ddot H^\#)
				\in \mathbb{RB}_{+}^{M\times l},\qquad
				\nabla_{\ddot H}F(\ddot W^\#,\ddot H^\#)
				\in \mathbb{RB}_{+\rbj}^{l\times N},\\
				&\operatorname{Re}(\ddot W^\#)
				\otimes
				\operatorname{Re}\!\left(
				\nabla_{\ddot W}F(\ddot W^\#,\ddot H^\#)
				\right)=0,\\
				&\operatorname{Im}_{\eta}(\ddot W^\#)
				\otimes
				\operatorname{Im}_{\eta}\!\left(
				\nabla_{\ddot W}F(\ddot W^\#,\ddot H^\#)
				\right)=0,
				\qquad \eta=\rbi,\rbj,\rbk,\\
				&\operatorname{Re}(\ddot H^\#)
				\otimes
				\operatorname{Re}\!\left(
				\nabla_{\ddot H}F(\ddot W^\#,\ddot H^\#)
				\right)=0,\\
				&\operatorname{Im}_{\rbj}(\ddot H^\#)
				\otimes
				\operatorname{Im}_{\rbj}\!\left(
				\nabla_{\ddot H}F(\ddot W^\#,\ddot H^\#)
				\right)=0.
			\end{aligned}
			\right.
		\end{equation}
	\end{proposition}
	
	\begin{proof}
		Under the real component representations \(\Phi\) and \(\Phi_{\rbj}\), the feasible sets become non-negative orthants in finite-dimensional Euclidean spaces. For a differentiable optimization problem over a non-negative orthant, the first-order stationarity condition is equivalent to the component-wise non-negativity of the gradient and the complementarity condition between each variable and the corresponding gradient component. Rewriting these conditions in RB form gives~\eqref{eq:kkt_component}. This completes the proof.
	\end{proof}
	
	The stationarity conditions in \eqref{eq:kkt_component} can also be written in the following variational inequality form.
	
	\begin{corollary}
		\label{cor:variational}
		The conditions in~\eqref{eq:kkt_component} are equivalent to
		\begin{equation}
			\label{eq:stationarity_vi}
			\left\{
			\begin{aligned}
				&
				\operatorname{Re}\!\left(\left\langle
				\nabla_{\ddot W}F(\ddot W^\#,\ddot H^\#),
				\ddot Z_W-\ddot W^\#
				\right\rangle\right)
				\ge 0,
				\quad
				\forall \ddot Z_W\in \mathbb{RB}_{+}^{M\times l},\\
				&
				\operatorname{Re}\!\left(\left\langle
				\nabla_{\ddot H}F(\ddot W^\#,\ddot H^\#),
				\ddot Z_H-\ddot H^\#
				\right\rangle\right)
				\ge 0,
				\quad
				\forall \ddot Z_H\in \mathbb{RB}_{+\rbj}^{l\times N}.
			\end{aligned}
			\right.
		\end{equation}
	\end{corollary}
	
	\begin{proposition}
		\label{prop:lower_bound}
		Assume that the affinity matrix \(A\) is symmetric and elementwise non-negative, and that \(\lambda\ge 0\). Then \(F(\ddot W,\ddot H)\) is bounded below on \(\mathbb{RB}_{+}^{M\times l}\times\mathbb{RB}_{+\rbj}^{l\times N}\).
	\end{proposition}
	
	\begin{proof}
		The reconstruction term \(\frac{1}{2}\|\ddot X-\ddot W\ddot H\|_F^2\) is non-negative. Since \(A\) is symmetric and elementwise non-negative, the graph Laplacian \(L=D-A\) is symmetric positive semidefinite. For any real matrix \(S=[\mathbf s_1,\ldots,\mathbf s_N]\),
		\[
		\operatorname{tr}(SLS^\top)
		=
		\frac{1}{2}
		\sum_{a,b=1}^{N} A_{ab}\|\mathbf s_a-\mathbf s_b\|_2^2
		\ge 0.
		\]
		Thus \(\operatorname{tr}(H_0LH_0^\top)\ge 0\) and \(\operatorname{tr}(H_2LH_2^\top)\ge 0\). Since \(\lambda\ge 0\), it follows that \(F(\ddot W,\ddot H)\ge 0\). Therefore, \(F\) is bounded below on the feasible set. This completes the proof.
	\end{proof}
	
	\begin{lemma}
		\label{lem:projection_descent}
		Let \(\mathcal C\) be a nonempty closed convex set in a finite-dimensional Euclidean space. For \(z\in\mathcal C\), \(g\) in the same space, and \(\tau>0\), define \(z^+=P_{\mathcal C}(z-\tau g)\). Then
		\[
		\langle g,z^+-z\rangle
		\le
		-\frac{1}{\tau}\|z^+-z\|^2
		\le 0.
		\]
	\end{lemma}
	
	\begin{proof}
		By the optimality condition of Euclidean projection,
		\[
		\left\langle z^+-(z-\tau g),y-z^+\right\rangle\ge 0,
		\qquad
		\forall y\in\mathcal C.
		\]
		Taking \(y=z\in\mathcal C\) gives
		\[
		\left\langle z^+-z+\tau g,z-z^+\right\rangle\ge 0.
		\]
		Equivalently,
		\[
		-\|z^+-z\|^2-\tau\langle g,z^+-z\rangle\ge 0,
		\]
		which proves the result. This completes the proof.
	\end{proof}
	
	\begin{lemma}
		\label{lem:block_lipschitz}
		For fixed \(\ddot H=H_0+H_2\rbj\), the block gradient
		\(\nabla_{\ddot W}F(\cdot,\ddot H)\) is Lipschitz continuous, and its Lipschitz constant can be taken as
		\[
		L_{\ddot W}(\ddot H)
		=
		2\left(\|H_0\|_2+\|H_2\|_2\right)^2.
		\]
		For fixed \(\ddot W=W_0+W_1\rbi+W_2\rbj+W_3\rbk\), the block gradient
		\(\nabla_{\ddot H}F(\ddot W,\cdot)\) is Lipschitz continuous, and its Lipschitz constant can be taken as
		\[
		L_{\ddot H}(\ddot W)
		=
		2
		\left(
		\|W_0\|_2+\|W_1\|_2+\|W_2\|_2+\|W_3\|_2
		\right)^2
		+\lambda\|L\|_2.
		\]
	\end{lemma}
	
	\begin{proof}
		All the following arguments are understood under the real component
		representation introduced above. Thus, the Frobenius norm and the corresponding
		inner product are those induced by the associated real vector spaces.
		
		First, fix \(\ddot H=H_0+H_2\rbj\) and consider the linear operator
		\[
		\mathcal T_{\ddot H}:\ddot W\mapsto \ddot W\ddot H.
		\]
		For any perturbation
		\(\Delta\ddot W=\Delta W_0+\Delta W_1\rbi+\Delta W_2\rbj+\Delta W_3\rbk\), the multiplication
		rules of reduced biquaternions imply
		\[
		\|\mathcal T_{\ddot H}(\Delta\ddot W)\|_F
		=
		\|\Delta\ddot W\,\ddot H\|_F
		\le
		\left(\|H_0\|_2+\|H_2\|_2\right)
		\|\Delta\ddot W\|_F.
		\]
		Hence,
		\[
		\|\mathcal T_{\ddot H}\|
		\le
		\|H_0\|_2+\|H_2\|_2.
		\]
		For fixed \(\ddot H\), the reconstruction term can be written as
		\[
		\frac{1}{2}\|\mathcal T_{\ddot H}(\ddot W)-\ddot X\|_F^2.
		\]
		Therefore, for any two matrices \(\ddot W_1\) and \(\ddot W_2\), with
		\(\Delta\ddot W=\ddot W_1-\ddot W_2\), the difference of the corresponding
		block gradients satisfies
		\[
		\nabla_{\ddot W}F(\ddot W_1,\ddot H)
		-
		\nabla_{\ddot W}F(\ddot W_2,\ddot H)
		=
		\mathcal T_{\ddot H}^{*}
		\mathcal T_{\ddot H}(\Delta\ddot W),
		\]
		where \(\mathcal T_{\ddot H}^{*}\) denotes the adjoint operator with respect to
		the real component inner product. Consequently,
		\[
		\begin{aligned}
			\left\|
			\nabla_{\ddot W}F(\ddot W_1,\ddot H)
			-
			\nabla_{\ddot W}F(\ddot W_2,\ddot H)
			\right\|_F
			&\le
			\|\mathcal T_{\ddot H}\|^2
			\|\Delta\ddot W\|_F  \\
			&\le
			\left(\|H_0\|_2+\|H_2\|_2\right)^2
			\|\ddot W_1-\ddot W_2\|_F.
		\end{aligned}
		\]
		It follows that \(2(\|H_0\|_2+\|H_2\|_2)^2\) is a Lipschitz modulus for
		\(\nabla_{\ddot W}F(\cdot,\ddot H)\).
		
		Next, fix
		\(\ddot W=W_0+W_1\rbi+W_2\rbj+W_3\rbk\) and consider the linear operator
		\[
		\mathcal S_{\ddot W}:\ddot H\mapsto \ddot W\ddot H
		\]
		on the structured space \(\mathbb{RB}_{+\rbj}^{l\times N}\). For any perturbation
		\(\Delta\ddot H=\Delta H_0+\Delta H_2\rbj\), the multiplication rules give
		\[
		\|\mathcal S_{\ddot W}(\Delta\ddot H)\|_F
		=
		\|\ddot W\,\Delta\ddot H\|_F
		\le
		\sqrt{2}
		\left(
		\|W_0\|_2+\|W_1\|_2+\|W_2\|_2+\|W_3\|_2
		\right)
		\|\Delta\ddot H\|_F.
		\]
		Thus,
		\[
		\|\mathcal S_{\ddot W}\|^2
		\le
		2
		\left(
		\|W_0\|_2+\|W_1\|_2+\|W_2\|_2+\|W_3\|_2
		\right)^2.
		\]
		For any two structured matrices
		\(\ddot H_1=H_{0,1}+H_{2,1}\rbj\) and
		\(\ddot H_2=H_{0,2}+H_{2,2}\rbj\), let
		\[
		\Delta\ddot H=\ddot H_1-\ddot H_2=\Delta H_0+\Delta H_2\rbj.
		\]
		The difference of the reconstruction-gradient part is given by
		\[
		\mathcal S_{\ddot W}^{*}\mathcal S_{\ddot W}(\Delta\ddot H).
		\]
		Moreover, the graph regularization term contributes the gradient difference
		\[
		\lambda \Delta H_0L+\lambda \Delta H_2L\rbj.
		\]
		Using the submultiplicativity of the spectral norm, we obtain
		\[
		\begin{aligned}
			\|\lambda \Delta H_0L+\lambda \Delta H_2L\rbj\|_F
			&=
			\lambda
			\left(
			\|\Delta H_0L\|_F^2+\|\Delta H_2L\|_F^2
			\right)^{1/2} \\
			&\le
			\lambda\|L\|_2
			\left(
			\|\Delta H_0\|_F^2+\|\Delta H_2\|_F^2
			\right)^{1/2} \\
			&=
			\lambda\|L\|_2\|\Delta\ddot H\|_F.
		\end{aligned}
		\]
		Combining the reconstruction part and the graph regularization part yields
		\[
		\begin{aligned}
			&
			\left\|
			\nabla_{\ddot H}F(\ddot W,\ddot H_1)
			-
			\nabla_{\ddot H}F(\ddot W,\ddot H_2)
			\right\|_F
			\le
			\left[
			2
			\left(
			\|W_0\|_2+\|W_1\|_2+\|W_2\|_2+\|W_3\|_2
			\right)^2
			+\lambda\|L\|_2
			\right]
			\|\ddot H_1-\ddot H_2\|_F.
		\end{aligned}
		\]
		Therefore, \(L_{\ddot H}(\ddot W)\) is a Lipschitz modulus for
		\(\nabla_{\ddot H}F(\ddot W,\cdot)\). The above Lipschitz constants are not necessarily the smallest possible constants, but they are valid upper bounds sufficient for the convergence analysis. This completes the proof.
	\end{proof}
	
	\begin{lemma}
		\label{lem:armijo_well_defined}
		At each iteration of Algorithm~\ref{alg:gnrbmf}, the Armijo line search procedures for both the \(\ddot W\)-update and the \(\ddot H\)-update terminate after finitely many reductions.
	\end{lemma}
	
	\begin{proof}
		It is sufficient to prove the claim for the \(\ddot W\)-update, since the proof for the \(\ddot H\)-update is analogous. For fixed \(\ddot H^t\), let \(\widehat L_{\ddot W}^t>0\) be a positive upper bound of a Lipschitz constant of \(\nabla_{\ddot W}F(\cdot,\ddot H^t)\). For a trial step size \(\alpha>0\), define
		\[
		\ddot W^+
		=
		P_{\mathbb{RB}_{+}^{M\times l}}
		\left[
		\ddot W^t-\alpha\nabla_{\ddot W}F(\ddot W^t,\ddot H^t)
		\right].
		\]
		By the descent lemma,
		\[
		\begin{aligned}
			F(\ddot W^+,\ddot H^t)-F(\ddot W^t,\ddot H^t)
			&\le
			\operatorname{Re}\!\left(\left\langle
			\nabla_{\ddot W}F(\ddot W^t,\ddot H^t),
			\ddot W^+-\ddot W^t
			\right\rangle\right)
			+
			\frac{\widehat L_{\ddot W}^t}{2}
			\|\ddot W^+-\ddot W^t\|_F^2.
		\end{aligned}
		\]
		By Lemma~\ref{lem:projection_descent},
		\[
		\operatorname{Re}\!\left(\left\langle
		\nabla_{\ddot W}F(\ddot W^t,\ddot H^t),
		\ddot W^+-\ddot W^t
		\right\rangle\right)
		\le
		-\frac{1}{\alpha}\|\ddot W^+-\ddot W^t\|_F^2.
		\]
		Therefore,
		\[
		\begin{aligned}
			&F(\ddot W^+,\ddot H^t)-F(\ddot W^t,\ddot H^t)
			-
			\sigma
			\operatorname{Re}\!\left(\left\langle
			\nabla_{\ddot W}F(\ddot W^t,\ddot H^t),
			\ddot W^+-\ddot W^t
			\right\rangle\right) \le
			\left(
			-\frac{1-\sigma}{\alpha}
			+\frac{\widehat L_{\ddot W}^t}{2}
			\right)
			\|\ddot W^+-\ddot W^t\|_F^2.
		\end{aligned}
		\]
		Hence the Armijo condition~\eqref{eq:armijo_W} is satisfied whenever \(\alpha\le 2(1-\sigma)/\widehat L_{\ddot W}^t\). Since, in the theoretical algorithm, the backtracking sequence starts from \(1\) and then follows \(1,\mu,\mu^2,\ldots\), which tends to zero, such a step size is reached after finitely many reductions. The proof for the \(\ddot H\)-update follows in the same way by Lemma~\ref{lem:block_lipschitz}. This completes the proof.
	\end{proof}
	
	\begin{theorem}
		\label{thm:descent}
		Let \(\{(\ddot W^t,\ddot H^t)\}_{t\ge 0}\) be the sequence generated by the infinite version of Algorithm~\ref{alg:gnrbmf}. Then \(\{F(\ddot W^t,\ddot H^t)\}_{t\ge 0}\) is nonincreasing and converges to a finite value. Moreover,
		\[
		\|\ddot W^{t+1}-\ddot W^t\|_F\to 0,
		\qquad
		\|\ddot H^{t+1}-\ddot H^t\|_F\to 0.
		\]
	\end{theorem}
	
	\begin{proof}
		By the Armijo condition~\eqref{eq:armijo_W} and Lemma~\ref{lem:projection_descent},
		\[
		\begin{aligned}
			F(\ddot W^{t+1},\ddot H^t)-F(\ddot W^t,\ddot H^t)
			&\le
			\sigma
			\operatorname{Re}\!\left(\left\langle
			\nabla_{\ddot W}F(\ddot W^t,\ddot H^t),
			\ddot W^{t+1}-\ddot W^t
			\right\rangle\right) \\
			&\le
			-\frac{\sigma}{\alpha_t}
			\|\ddot W^{t+1}-\ddot W^t\|_F^2
			\le 0.
		\end{aligned}
		\]
		Similarly,
		\[
		\begin{aligned}
			F(\ddot W^{t+1},\ddot H^{t+1})
			-F(\ddot W^{t+1},\ddot H^t)
			&\le
			-\frac{\sigma}{\beta_t}
			\|\ddot H^{t+1}-\ddot H^t\|_F^2
			\le 0.
		\end{aligned}
		\]
		Combining these two inequalities yields
		\[
		F(\ddot W^{t+1},\ddot H^{t+1})
		\le
		F(\ddot W^t,\ddot H^t).
		\]
		Thus, the objective values are nonincreasing. By Proposition~\ref{prop:lower_bound}, they are bounded below and therefore converge to a finite value.
		
		Since the theoretical Armijo searches start from the trial step size \(1\), we have \(\alpha_t=\mu^{d_t}\le 1\) and \(\beta_t=\mu^{s_t}\le 1\). Hence \(1/\alpha_t\ge 1\) and \(1/\beta_t\ge 1\). Therefore,
		\[
		\begin{aligned}
			F(\ddot W^{t+1},\ddot H^{t+1})
			&\le
			F(\ddot W^t,\ddot H^t)
			-\sigma\|\ddot W^{t+1}-\ddot W^t\|_F^2
			-\sigma\|\ddot H^{t+1}-\ddot H^t\|_F^2.
		\end{aligned}
		\]
		Summing this inequality from \(t=0\) to \(T\) gives
		\[
		\begin{aligned}
			\sigma
			\sum_{t=0}^{T}
			\left(
			\|\ddot W^{t+1}-\ddot W^t\|_F^2
			+\|\ddot H^{t+1}-\ddot H^t\|_F^2
			\right)
			\le
			F(\ddot W^0,\ddot H^0)
			-F(\ddot W^{T+1},\ddot H^{T+1}).
		\end{aligned}
		\]
		Letting \(T\to\infty\), the right-hand side remains bounded. Therefore,
		\[
		\sum_{t=0}^{\infty}\|\ddot W^{t+1}-\ddot W^t\|_F^2<\infty,
		\qquad
		\sum_{t=0}^{\infty}\|\ddot H^{t+1}-\ddot H^t\|_F^2<\infty.
		\]
		Consequently, \(\|\ddot W^{t+1}-\ddot W^t\|_F\to 0\) and \(\|\ddot H^{t+1}-\ddot H^t\|_F\to 0\). This completes the proof.
	\end{proof}
	
	\begin{lemma}
		\label{lem:stepsize_lower_bound}
		Suppose that the sequence \(\{(\ddot W^t,\ddot H^t)\}_{t\ge 0}\) generated by Algorithm~\ref{alg:gnrbmf} converges. Then there exist constants \(\underline{\alpha}>0\), \(\underline{\beta}>0\), and \(t_0\) such that \(\alpha_t\ge \underline{\alpha}\) and \(\beta_t\ge \underline{\beta}\) for all \(t\ge t_0\).
	\end{lemma}
	
	\begin{proof}
		Since the sequence converges, it is bounded. By Lemma~\ref{lem:block_lipschitz}, there exists \(\overline L_{\ddot W}>0\) such that, for all sufficiently large \(t\),
		\[
		L_{\ddot W}(\ddot H^t)\le \overline L_{\ddot W}.
		\]
		From the proof of Lemma~\ref{lem:armijo_well_defined}, the Armijo condition for the \(\ddot W\)-update is guaranteed whenever \(\alpha\le 2(1-\sigma)/\overline L_{\ddot W}\). Since the backtracking sequence starts from \(1\) and is reduced by \(\mu\), the accepted step size satisfies
		\[
		\alpha_t
		\ge
		\mu
		\min
		\left\{
		1,\frac{2(1-\sigma)}{\overline L_{\ddot W}}
		\right\}
		\triangleq
		\underline{\alpha}>0
		\]
		for all sufficiently large \(t\).
		
		Similarly, since \(\{\ddot W^{t+1}\}_{t\ge 0}\) is bounded, there exists \(\overline L_{\ddot H}>0\) such that \(L_{\ddot H}(\ddot W^{t+1})\le \overline L_{\ddot H}\) for all sufficiently large \(t\). Hence
		\[
		\beta_t
		\ge
		\mu
		\min
		\left\{
		1,\frac{2(1-\sigma)}{\overline L_{\ddot H}}
		\right\}
		\triangleq
		\underline{\beta}>0.
		\]
		This completes the proof.
	\end{proof}
	
	\begin{theorem}
		\label{thm:stationary}
		Let \(\{(\ddot W^t,\ddot H^t)\}_{t\ge 0}\) be the sequence generated by the infinite version of Algorithm~\ref{alg:gnrbmf}. If
		\[
		\lim_{t\to\infty}\ddot W^t=\ddot W^\#,
		\qquad
		\lim_{t\to\infty}\ddot H^t=\ddot H^\#,
		\]
		then \((\ddot W^\#,\ddot H^\#)\) is a stationary point of problem~\eqref{eq:gnrbmf_model}.
	\end{theorem}
	
	\begin{proof}
		Since every iterate is feasible and the feasible sets are closed, \(\ddot W^\#\in \mathbb{RB}_{+}^{M\times l}\) and \(\ddot H^\#\in \mathbb{RB}_{+\rbj}^{l\times N}\).
		
		We first consider the \(\ddot W\)-block. From the projection update~\eqref{eq:W_update}, the optimality condition of Euclidean projection gives, for any fixed \(\ddot Z_W\in \mathbb{RB}_{+}^{M\times l}\),
		\[
		\begin{aligned}
			&\operatorname{Re}\!\left(\left\langle
			\ddot W^{t+1}
			-\left[
			\ddot W^t
			-\alpha_t\nabla_{\ddot W}F(\ddot W^t,\ddot H^t)
			\right],
			\ddot Z_W-\ddot W^{t+1}
			\right\rangle\right)
			\ge 0.
		\end{aligned}
		\]
		Equivalently,
		\[
		\begin{aligned}
			&\operatorname{Re}\!\left(\left\langle
			\nabla_{\ddot W}F(\ddot W^t,\ddot H^t),
			\ddot Z_W-\ddot W^{t+1}
			\right\rangle\right) \ge
			\frac{1}{\alpha_t}
			\operatorname{Re}\!\left(\left\langle
			\ddot W^t-\ddot W^{t+1},
			\ddot Z_W-\ddot W^{t+1}
			\right\rangle\right).
		\end{aligned}
		\]
		By Lemma~\ref{lem:stepsize_lower_bound}, \(\alpha_t\) is bounded away from zero for all sufficiently large \(t\). Thus \(1/\alpha_t\) is bounded above. By Theorem~\ref{thm:descent}, \(\ddot W^{t+1}-\ddot W^t\to 0\). Moreover, since \(\ddot W^{t+1}\to \ddot W^\#\), the sequence \(\{\ddot Z_W-\ddot W^{t+1}\}\) is bounded for each fixed feasible \(\ddot Z_W\). Hence the right-hand side tends to zero. Using the continuity of \(\nabla_{\ddot W}F\) and taking the limit gives
		\[
		\operatorname{Re}\!\left(\left\langle
		\nabla_{\ddot W}F(\ddot W^\#,\ddot H^\#),
		\ddot Z_W-\ddot W^\#
		\right\rangle\right)
		\ge 0,
		\qquad
		\forall \ddot Z_W\in \mathbb{RB}_{+}^{M\times l}.
		\]
		
		Next, consider the \(\ddot H\)-block. From the projection update~\eqref{eq:H_update}, for any fixed \(\ddot Z_H\in \mathbb{RB}_{+\rbj}^{l\times N}\),
		\[
		\begin{aligned}
			&\operatorname{Re}\!\left(\left\langle
			\nabla_{\ddot H}F(\ddot W^{t+1},\ddot H^t),
			\ddot Z_H-\ddot H^{t+1}
			\right\rangle\right) \ge
			\frac{1}{\beta_t}
			\operatorname{Re}\!\left(\left\langle
			\ddot H^t-\ddot H^{t+1},
			\ddot Z_H-\ddot H^{t+1}
			\right\rangle\right).
		\end{aligned}
		\]
		By Lemma~\ref{lem:stepsize_lower_bound}, \(\beta_t\) is bounded away from zero for all sufficiently large \(t\). Thus \(1/\beta_t\) is bounded above. By Theorem~\ref{thm:descent}, \(\ddot H^{t+1}-\ddot H^t\to 0\). Moreover, since \(\ddot H^{t+1}\to \ddot H^\#\), the sequence \(\{\ddot Z_H-\ddot H^{t+1}\}\) is bounded for each fixed feasible \(\ddot Z_H\). Also, \((\ddot W^{t+1},\ddot H^t)\to(\ddot W^\#,\ddot H^\#)\). Therefore, by the continuity of \(\nabla_{\ddot H}F\) and by taking the limit, we obtain
		\[
		\operatorname{Re}\!\left(\left\langle
		\nabla_{\ddot H}F(\ddot W^\#,\ddot H^\#),
		\ddot Z_H-\ddot H^\#
		\right\rangle\right)
		\ge 0,
		\qquad
		\forall \ddot Z_H\in \mathbb{RB}_{+\rbj}^{l\times N}.
		\]
		Thus \((\ddot W^\#,\ddot H^\#)\) satisfies the variational inequalities in Corollary~\ref{cor:variational}. Hence it is a stationary point of problem~\eqref{eq:gnrbmf_model}. This completes the proof.
	\end{proof}
	
	\begin{remark}
		Theorem~\ref{thm:descent} shows that the objective values generated by Algorithm~\ref{alg:gnrbmf} are monotonically nonincreasing and convergent. Theorem~\ref{thm:stationary} further shows that if the generated sequence itself converges, then its limit is a stationary point of problem~\eqref{eq:gnrbmf_model}. Since the GNRBMF problem is non-convex with respect to \((\ddot W,\ddot H)\), the above result does not imply convergence to a global minimizer.
	\end{remark}
	
	\section{Experiments}
	\label{sec:experiments}
	
	In this section, numerical experiments on color image recognition are conducted to evaluate the proposed GNRBMF method. The experiments are designed from three aspects. First, GNRBMF is compared with NRBMF to examine the effect of graph regularization. Second, GNRBMF is compared with real-valued, quaternion-based, and reduced-biquaternion-based methods under different factorization ranks. Third, the influence of the regularization parameter \(\lambda\), the number of nearest neighbors \(p\), and the numerical convergence behavior of the proposed algorithm are investigated.
	
	All experiments are run in MATLAB 2024a under Windows 10 on a laptop
	with a 2.30GHz CPU and 16GB of memory.
	
	\subsection{Datasets and Experimental Setup}
	\label{subsec:datasets_and_setup}
	
	The experiments are conducted on three color image datasets, including CASIA-FaceV5, KDEF, and Asirra. The details of these datasets and the experimental settings are described below. Some sample images are shown in Fig.~\ref{fig:dataset_samples}.
	
	The CASIA-FaceV5 dataset\footnote{\url{https://english.ia.cas.cn/db/201610/t20161026_169405.html}} contains 2500 color face images from 500 subjects, with five images for each subject. In our experiment, 120 subjects are randomly selected. For each subject, three images are used for training and the remaining two images are used for testing. Thus, each trial contains 360 training images and 240 test images.
	
	The KDEF dataset\footnote{\url{https://www.kdef.se/}} is a color facial expression dataset developed at Karolinska Institutet. It contains 4900 images of 70 individuals, including 35 females and 35 males. Each individual displays seven emotional expressions, and each expression is photographed from five view angles in two sessions. Compared with CASIA-FaceV5, KDEF introduces more intra-class variations caused by facial expression and pose changes. In our experiment, 30 individuals are randomly selected. For each selected individual, 23 images are used for training and 12 images are used for testing. Thus, each trial contains 690 training images and 360 test images.
	
	The Asirra dataset\footnote{\url{https://www.microsoft.com/en-us/download/details.aspx?id=54765}} is a cat-and-dog image dataset derived from Microsoft Research and Petfinder image resources. Compared with face images, these natural images contain more complex variations in background, pose, illumination, texture, and object scale. In our experiment, 250 cat images and 250 dog images are selected. For each class, 130 images are used for training and 120 images are used for testing. Thus, each trial contains 260 training images and 240 test images.
	
	\begin{figure}[!htbp]
		\centering
		\begin{tabular}{ccc}
			\includegraphics[width=0.30\textwidth]{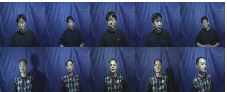} &
			\includegraphics[width=0.30\textwidth]{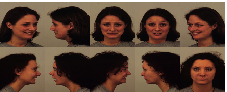} &
			\includegraphics[width=0.30\textwidth]{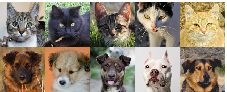} \\
			(a) CASIA-FaceV5 & (b) KDEF & (c) Asirra
		\end{tabular}
		\caption{Sample images from the three color image datasets.}
		\label{fig:dataset_samples}
	\end{figure}
	
	In all experiments, each color image is resized to \(80\times 100\) pixels. Following the full RB representation used in NRBMF \cite{miao2024non}, each color image is represented as
	\[
	\ddot X=X_{\mathrm{av}}+X_R \rbi+X_G \rbj+X_B \rbk,
	\]
	where \(X_R\), \(X_G\), and \(X_B\) denote the red, green, and blue channels, respectively, and
	\[
	X_{\mathrm{av}}=\frac{X_R+X_G+X_B}{3}
	\]
	is used as the real part. This representation preserves the three color channels in the imaginary components and introduces an average intensity component in the real part.
	
	Suppose that the training set contains \(N_{\mathrm{tr}}\) images and the test set contains \(N_{\mathrm{te}}\) images. After RB representation and column-wise vectorization, the training and test data matrices are denoted by \(\ddot X\in \mathbb{RB}_{+}^{M\times N_{\mathrm{tr}}}\) and \(\ddot Y\in \mathbb{RB}_{+}^{M\times N_{\mathrm{te}}}\), respectively, where \(M\) is the dimension of each vectorized RB sample.
	
	For each training set, GNRBMF is first applied to learn the basis matrix \(\ddot W\in \mathbb{RB}_{+}^{M\times l}\) and the coefficient matrix \(\ddot H=H_0+H_2\rbj\in \mathbb{RB}_{+\rbj}^{l\times N_{\mathrm{tr}}}\), where \(l\) is the factorization rank. The columns of \(\ddot H\) are used as the low-dimensional representations of the training samples. For a test sample \(\ddot{\mathbf y}_s\), its coefficient vector is obtained by solving
	\[
	\min_{\ddot{\mathbf h}\in \mathbb{RB}_{+\rbj}^{l\times 1}}
	\|\ddot{\mathbf y}_s-\ddot W\ddot{\mathbf h}\|_F^2.
	\]
	This subproblem is solved by the same projected gradient strategy with the obtained basis matrix fixed. Thus, both training and test samples are represented in the same coefficient space.
	
	For classification, the nearest-neighbor rule based on cosine similarity is used in the coefficient space. Since \(\ddot H=H_0+H_2\rbj\), the feature vector of the \(k\)-th training sample is defined as
	\[
	\mathbf z_k^{(\mathrm{train})}
	=
	\left[
	\mathbf h_{0,k}^{\top},
	\mathbf h_{2,k}^{\top}
	\right]^{\top},
	\]
	where \(\mathbf h_{0,k}\) and \(\mathbf h_{2,k}\) are the \(k\)-th columns of \(H_0\) and \(H_2\), respectively. Similarly, the feature vector of the \(s\)-th test sample is denoted by \(\mathbf z_s^{(\mathrm{test})}\). The cosine similarity between the \(s\)-th test sample and the \(k\)-th training sample is calculated as
	\[
	d_{s,k}
	=
	\frac{
		\left\langle \mathbf z_k^{(\mathrm{train})},\mathbf z_s^{(\mathrm{test})}\right\rangle
	}
	{
		\left\|\mathbf z_k^{(\mathrm{train})}\right\|_2
		\left\|\mathbf z_s^{(\mathrm{test})}\right\|_2
	},
	\qquad
	k=1,2,\ldots,N_{\mathrm{tr}}.
	\]
	The recognition rate is used as the evaluation metric and is defined as
	\[
	\mathrm{Rec}
	=
	\frac{N_{\mathrm{correct}}}{N_{\mathrm{te}}}
	\times 100\%,
	\]
	where \(N_{\mathrm{correct}}\) denotes the number of correctly classified test samples and \(N_{\mathrm{te}}\) denotes the total number of test samples. A higher recognition rate indicates better classification performance.
	
	For graph-based methods, namely GNRBMF and GRIPG, the symmetric \(0\)--\(1\) nearest-neighbor graph is constructed from the training samples according to~\eqref{eq:affinity_general}. The factorization rank is selected from \(l\in\{5,10,15,20,25\}\). Since only GNRBMF and GRIPG involve graph-related parameters, the regularization parameter \(\lambda\) and the number of nearest neighbors \(p\) are selected from predefined candidate sets. Specifically, the regularization parameter is selected from \(\lambda \in \{10^{-3},10^{-2},10^{-1},1,10,100\}\), and the number of nearest neighbors is selected from \(p\in \{3,5,7,9,11\}\).
	
	For each dataset and each factorization rank, the graph-related parameters of GNRBMF and GRIPG are determined by selecting the best-performing values from the same candidate sets following an identical selection procedure. For methods without graph-related parameters, the recognition results are obtained under their standard settings. All recognition experiments are independently repeated ten times with random training/test splits, and the mean recognition rates with standard deviations are reported.
	
	The proposed method is compared with the following methods:
	\begin{itemize}
		\item GNRBMF: the proposed graph regularized non-negative reduced biquaternion matrix factorization method;
		\item GRIPG~\cite{cai2011graph,lin2007projected}: graph regularized real-valued NMF solved by an improved projected gradient scheme;
		\item NRBMF~\cite{miao2024non}: non-negative reduced biquaternion matrix factorization without graph regularization;
		\item QIPG~\cite{ke2023quasi}: quasi non-negative quaternion matrix factorization solved by an improved projected gradient scheme;
		\item QPCA~\cite{liu2022qpca}: quaternion principal component analysis for color image recognition;
		\item RIPG~\cite{lin2007projected}: real-valued NMF solved by an improved projected gradient scheme.
	\end{itemize}
	
	\subsection{Recognition Performance}
	\label{subsec:recognition_performance}
	
	Table~\ref{tab:recognition_results} reports the recognition rates of different methods on the CASIA-FaceV5, KDEF, and Asirra datasets under different factorization ranks. Each result is reported as the mean value and standard deviation over ten independent trials. For each factorization rank, the best result and the second-best result are highlighted in bold and underlined, respectively.
	
	\begin{table*}[!htbp]
		\centering
		\caption{Recognition rates (\%) of different methods on the three datasets.}
		\label{tab:recognition_results}
		\scriptsize
		\renewcommand{\arraystretch}{1.12}
		\resizebox{\textwidth}{!}{
			\begin{tabular}{llccccc}
				\toprule
				Dataset & Method & \(l=5\) & \(l=10\) & \(l=15\) & \(l=20\) & \(l=25\) \\
				\midrule
				
				\multirow{6}{*}{CASIA-FaceV5}
				& GNRBMF & \(\mathbf{70.37\pm4.50}\) & \(\mathbf{78.28\pm3.43}\) & \(\mathbf{83.20\pm3.27}\) & \(\mathbf{84.12\pm3.47}\) & \(\mathbf{85.28\pm3.09}\) \\
				& GRIPG  & \(67.33\pm3.87\) & \(\underline{76.92\pm2.37}\) & \(80.75\pm3.92\) & \(82.83\pm3.22\) & \(83.67\pm3.84\) \\
				& NRBMF  & \(\underline{68.58\pm2.87}\) & \(\underline{76.92\pm3.38}\) & \(\underline{81.33\pm4.55}\) & \(82.50\pm3.58\) & \(\underline{84.25\pm3.40}\) \\
				& QIPG   & \(65.83\pm3.95\) & \(72.50\pm4.07\) & \(76.50\pm3.20\) & \(79.17\pm1.84\) & \(81.50\pm3.88\) \\
				& QPCA   & \(57.67\pm1.27\) & \(59.25\pm1.75\) & \(63.00\pm1.85\) & \(66.75\pm4.96\) & \(68.42\pm4.19\) \\
				& RIPG   & \(66.50\pm4.11\) & \(76.58\pm2.52\) & \(79.75\pm3.21\) & \(\underline{83.42\pm3.44}\) & \(83.08\pm2.79\) \\
				
				\midrule
				
				\multirow{6}{*}{KDEF}
				& GNRBMF & \(\mathbf{87.18\pm3.98}\) & \(\mathbf{91.60\pm2.87}\) & \(\mathbf{94.18\pm2.61}\) & \(\mathbf{93.35\pm2.71}\) & \(\mathbf{94.18\pm2.71}\) \\
				& GRIPG  & \(\underline{86.58\pm4.09}\) & \(90.50\pm2.40\) & \(\underline{91.83\pm2.55}\) & \(\underline{92.42\pm2.79}\) & \(92.25\pm2.14\) \\
				& NRBMF  & \(86.25\pm4.28\) & \(\underline{90.58\pm2.90}\) & \(\underline{91.83\pm2.10}\) & \(92.33\pm2.14\) & \(92.75\pm3.42\) \\
				& QIPG   & \(85.92\pm4.18\) & \(90.08\pm3.59\) & \(91.50\pm3.01\) & \(92.33\pm2.68\) & \(\underline{93.00\pm2.54}\) \\
				& QPCA   & \(74.50\pm1.94\) & \(80.42\pm4.56\) & \(82.67\pm4.39\) & \(86.33\pm4.35\) & \(89.33\pm3.73\) \\
				& RIPG   & \(86.25\pm3.62\) & \(90.50\pm2.88\) & \(90.75\pm3.03\) & \(91.67\pm2.52\) & \(91.83\pm2.24\) \\
				
				\midrule
				
				\multirow{6}{*}{Asirra}
				& GNRBMF & \(\underline{59.58\pm2.34}\) & \(\mathbf{67.67\pm2.97}\) & \(\mathbf{71.58\pm2.11}\) & \(\mathbf{71.33\pm1.62}\) & \(\mathbf{71.50\pm1.09}\) \\
				& GRIPG  & \(58.58\pm2.31\) & \(\underline{66.17\pm2.63}\) & \(67.83\pm2.95\) & \(68.42\pm1.99\) & \(68.75\pm2.32\) \\
				& NRBMF  & \(58.33\pm2.19\) & \(62.75\pm3.54\) & \(\underline{69.00\pm4.42}\) & \(\underline{70.00\pm3.03}\) & \(\underline{69.75\pm2.73}\) \\
				& QIPG   & \(\mathbf{59.75\pm4.49}\) & \(64.25\pm1.92\) & \(66.42\pm2.85\) & \(68.42\pm4.36\) & \(65.75\pm5.63\) \\
				& QPCA   & \(53.42\pm3.60\) & \(53.33\pm2.26\) & \(56.50\pm1.85\) & \(56.92\pm2.87\) & \(60.17\pm2.93\) \\
				& RIPG   & \(58.67\pm1.94\) & \(64.50\pm2.19\) & \(66.58\pm2.07\) & \(68.42\pm2.36\) & \(69.00\pm3.18\) \\
				
				\bottomrule
			\end{tabular}
		}
	\end{table*}
	
	Several observations can be made from Table~\ref{tab:recognition_results}. First, GNRBMF consistently achieves the best recognition rates on the two face datasets. On CASIA-FaceV5, the recognition rate of GNRBMF increases from \(70.37\%\) at \(l=5\) to \(85.28\%\) at \(l=25\), and it obtains the best result under all tested ranks. On KDEF, GNRBMF also ranks first for all values of \(l\), reaching \(94.18\%\) at \(l=15\) and \(l=25\). These results show that the proposed graph regularized RB factorization model is effective for color face recognition, even when facial expression and pose variations are present.
	
	Second, GNRBMF performs better than NRBMF under all listed ranks on the three datasets. Since the two models use the same non-negative RB representation and mainly differ in the graph regularization term, this comparison indicates that preserving the local geometric structure of training samples can improve the discriminative ability of the obtained RB coefficient representation. The improvement is clear on CASIA-FaceV5 and KDEF, and it is also observed on Asirra, where the natural images are more complex than face images.
	
	Third, compared with GRIPG, GNRBMF also gives higher recognition rates under all listed ranks. Both methods use graph regularization, but GRIPG is based on a real-valued representation, whereas GNRBMF uses the RB representation to jointly encode the RGB channels in a unified algebraic form. This comparison suggests that combining graph regularization with the RB representation is beneficial for color image recognition.
	
	On Asirra, the recognition rates are generally lower than those on the two face datasets. This is reasonable because Asirra contains natural cat-and-dog images with more complex variations in background, pose, illumination, texture, and object scale. In this more challenging setting, GNRBMF obtains the best results for \(l=10,15,20,25\), and achieves the second-best result at \(l=5\). Overall, the recognition results demonstrate the effectiveness of introducing graph regularization into non-negative RB matrix factorization.
	
	\subsection{Parameter Sensitivity}
	\label{subsec:parameter_sensitivity}
	
	This subsection analyzes the sensitivity of GNRBMF to two graph-related parameters, namely the regularization parameter \(\lambda\) and the number of nearest neighbors \(p\). The parameter \(\lambda\) controls the contribution of the graph regularization term in the objective function, while \(p\) determines the local connectivity of the nearest-neighbor graph. To examine their effects, one parameter is fixed and the other parameter is varied. Specifically, \(p\) is first fixed as \(p=5\), and \(\lambda\) is selected from \(\lambda \in \{10^{-3},10^{-2},10^{-1},1,10,100\}\). Then, \(\lambda\) is fixed as \(\lambda=0.01\), and \(p\) is selected from \(p\in \{3,5,7,9,11\}\). In both experiments, the factorization rank is tested over \(l\in\{5,10,15,20,25\}\).
	
	Figure~\ref{fig:param_lambda} shows the recognition rates obtained by GNRBMF under different values of \(\lambda\) when \(p=5\). Each panel corresponds to one factorization rank.
	
	\begin{figure}[!htbp]
		\centering
		\begin{tabular}{ccc}
			\includegraphics[width=0.31\textwidth]{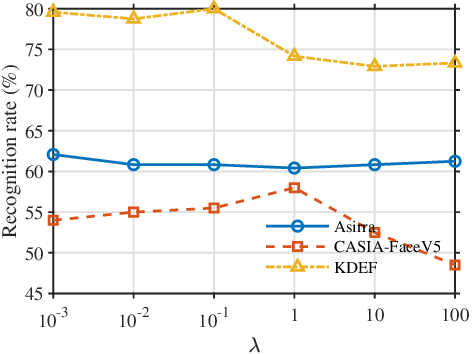} &
			\includegraphics[width=0.31\textwidth]{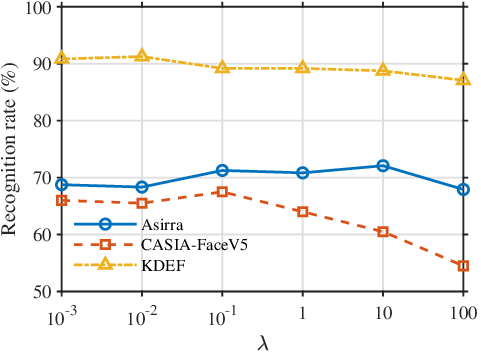} &
			\includegraphics[width=0.31\textwidth]{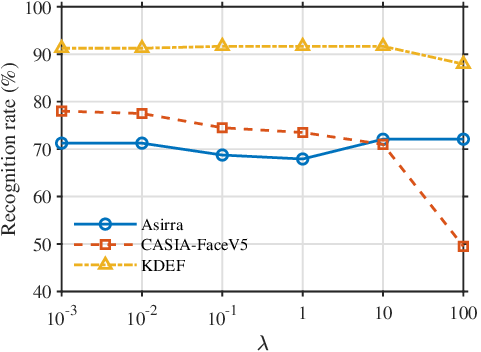} \\
			(a) \(l=5\) & (b) \(l=10\) & (c) \(l=15\) \\
			\multicolumn{3}{c}{%
				\begin{tabular}{cc}
					\includegraphics[width=0.31\textwidth]{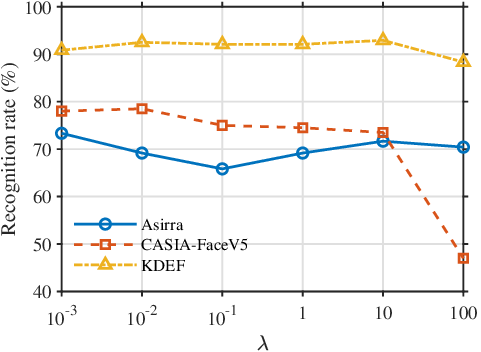} &
					\includegraphics[width=0.31\textwidth]{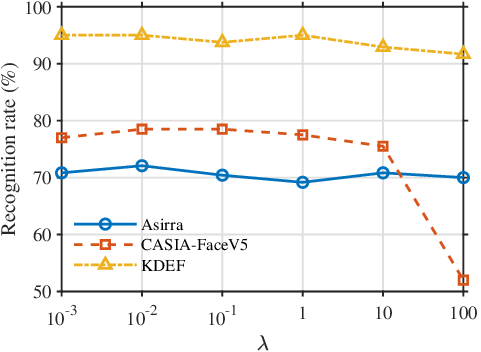} \\
					(d) \(l=20\) & (e) \(l=25\)
			\end{tabular}}
		\end{tabular}
		\caption{Recognition rates of GNRBMF with respect to \(\lambda\) under \(p=5\).}
		\label{fig:param_lambda}
	\end{figure}
	
	As shown in Fig.~\ref{fig:param_lambda}, the influence of \(\lambda\) varies across different datasets and factorization ranks. On CASIA-FaceV5, the recognition performance is relatively stable when \(\lambda\) takes small or moderate values. However, the performance decreases obviously when \(\lambda\) becomes too large, especially when \(\lambda=100\). This indicates that an excessively large graph regularization weight may overemphasize the local smoothness of the coefficient representation and weaken the discriminative differences among samples.
	
	On KDEF and Asirra, the recognition rates are relatively stable over the tested range of \(\lambda\). Although small fluctuations can still be observed under different ranks, no severe performance degradation appears in most cases. This suggests that the graph regularization term can provide useful local structure information within a relatively broad parameter range. Overall, a small-to-moderate value of \(\lambda\) is more suitable for balancing the reconstruction term and the graph regularization term. Therefore, \(\lambda=0.01\) is used as a representative value in the following \(p\)-sensitivity experiment.
	
	Figure~\ref{fig:param_p} shows the recognition rates obtained by GNRBMF under different values of \(p\) when \(\lambda=0.01\). Each panel corresponds to one factorization rank.
	
	\begin{figure}[!htbp]
		\centering
		\begin{tabular}{ccc}
			\includegraphics[width=0.31\textwidth]{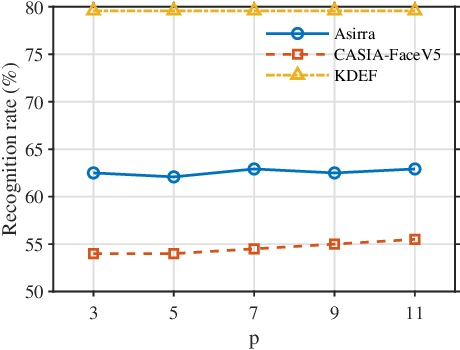} &
			\includegraphics[width=0.31\textwidth]{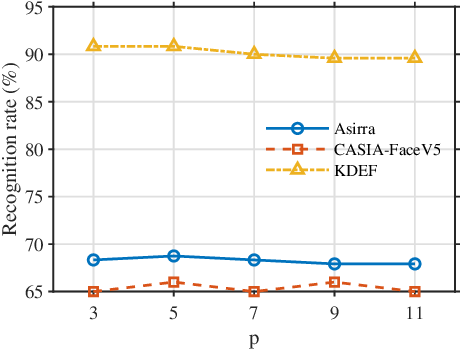} &
			\includegraphics[width=0.31\textwidth]{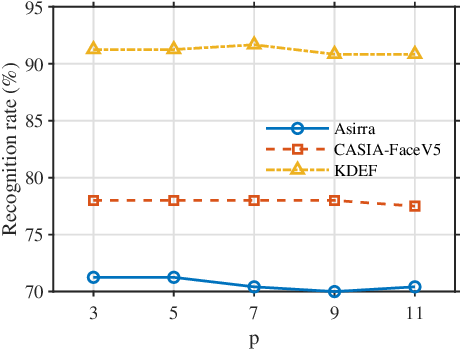} \\
			(a) \(l=5\) & (b) \(l=10\) & (c) \(l=15\) \\
			\multicolumn{3}{c}{%
				\begin{tabular}{cc}
					\includegraphics[width=0.31\textwidth]{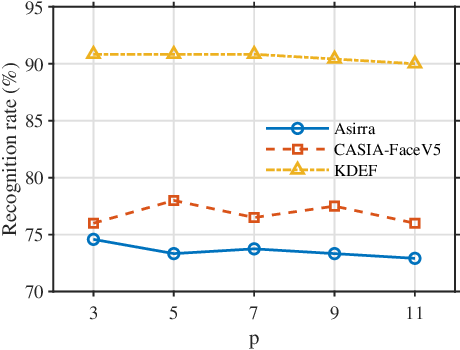} &
					\includegraphics[width=0.31\textwidth]{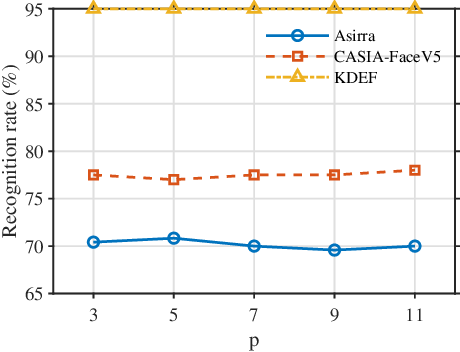} \\
					(d) \(l=20\) & (e) \(l=25\)
			\end{tabular}}
		\end{tabular}
		\caption{Recognition rates of GNRBMF with respect to \(p\) under \(\lambda=0.01\).}
		\label{fig:param_p}
	\end{figure}
	
	Compared with \(\lambda\), the recognition performance is generally less sensitive to \(p\), as shown in Fig.~\ref{fig:param_p}. Across the five factorization ranks, no obvious performance collapse is observed when \(p\) changes within the tested range \(\{3,5,7,9,11\}\). Although different datasets may prefer slightly different neighborhood sizes, the curves remain relatively stable in most cases. This indicates that GNRBMF is robust to the choice of neighborhood size once a reasonable nearest-neighbor graph has been constructed.
	
	In summary, the parameter sensitivity experiments show that \(\lambda\) has a more noticeable influence on GNRBMF, especially when an excessively large value is used. In contrast, the influence of \(p\) is relatively moderate within the tested range. These results suggest that GNRBMF can maintain stable recognition performance under a reasonable graph parameter setting.
	
	\subsection{Convergence Behavior}
	\label{subsec:convergence_behavior}
	
	To examine the numerical convergence behavior of the proposed algorithm, the objective value of GNRBMF is recorded with respect to the number of iterations on the three datasets. In this experiment, the factorization rank is fixed as \(l=15\), the regularization parameter is set to \(\lambda=0.01\), and the number of nearest neighbors is set to \(p=5\). At the \(t\)-th iteration, the objective value is calculated as
	\[
	F^{(t)}=F(\ddot W^{\,t},\ddot H^{\,t}),
	\]
	where \(F\) denotes the objective function in the GNRBMF model~\eqref{eq:gnrbmf_model}, and its component form is given in~\eqref{eq:objective_component_form}. Therefore, \(F^{(t)}\) contains both the reconstruction error term and the graph regularization term. For better visualization, the objective values are displayed in scientific notation when their magnitudes are large.
	
	\begin{figure}[!htbp]
		\centering
		\begin{tabular}{ccc}
			\includegraphics[width=0.31\textwidth]{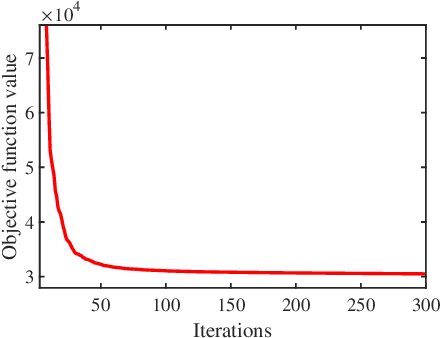} &
			\includegraphics[width=0.31\textwidth]{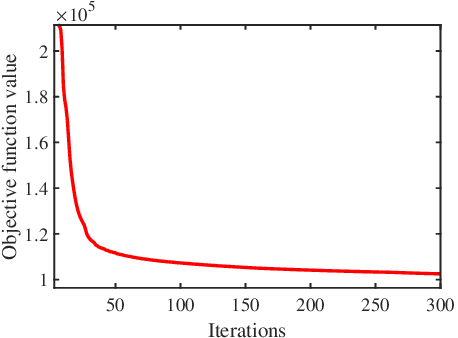} &
			\includegraphics[width=0.31\textwidth]{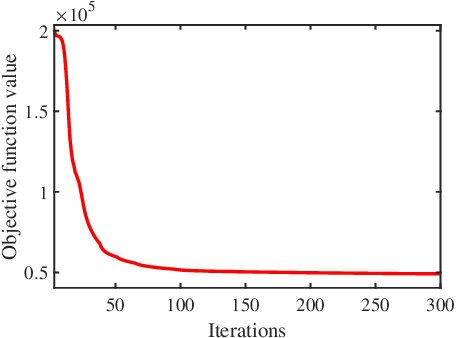} \\
			(a) CASIA-FaceV5 & (b) KDEF & (c) Asirra
		\end{tabular}
		\caption{Objective value versus iterations of GNRBMF on the three datasets.}
		\label{fig:convergence}
	\end{figure}
	
	As shown in Fig.~\ref{fig:convergence}, the objective values decrease rapidly during the first several iterations and then gradually approach stable values on all three datasets. The sharp decrease at the early stage indicates that the projected gradient updates can quickly reduce the reconstruction and graph regularization terms, while the later flat trend shows that the algorithm gradually reaches a stable optimization state. Similar decreasing patterns can be observed on CASIA-FaceV5, KDEF, and Asirra, which suggests that the optimization procedure is numerically stable for different types of color image data. These empirical observations are consistent with the convergence analysis given in Section~\ref{sec:optimization}.
	
	\section{Conclusion}
	\label{sec:conclusion}
	
	In this paper, a graph regularized non-negative reduced biquaternion matrix factorization model is proposed for color image recognition. By incorporating a graph Laplacian regularizer into the structured RB coefficient matrix, the proposed model retains the non-negativity property of NRBMF while preserving local geometric information in the obtained low-dimensional coefficient representations.
	
	To solve the resulting constrained optimization problem, we developed an alternating projected gradient algorithm with Armijo backtracking line search. The optimization procedure was derived in component form, which is consistent with the algebraic structure of reduced biquaternions and the imposed non-negativity constraints. The convergence analysis showed that the objective values generated by the proposed algorithm are nonincreasing and convergent. Moreover, if the generated sequence converges, then its limit point is a stationary point of the proposed GNRBMF optimization problem.
	
	Experiments on CASIA-FaceV5, KDEF, and Asirra demonstrated the effectiveness of GNRBMF\@. Compared with NRBMF and several real-valued, quaternion-based, and reduced-biquaternion-based methods, the proposed model achieves competitive or superior recognition performance in most tested settings. The results confirm that graph regularization can improve the discriminative ability of the obtained RB coefficients, while the parameter sensitivity and convergence experiments further demonstrate the stability of the proposed method.
	
	In our ongoing work, we are further investigating multiple-graph and auto-weighted graph regularization within the proposed framework. Such extensions may improve robustness by combining complementary neighborhood structures and reducing the dependence on a specific graph construction strategy. In addition, we will investigate more efficient optimization algorithms and apply the proposed model to broader color image analysis tasks.
	
	\backmatter
	
	%\bmhead{Acknowledgements}
	%This research was supported by the National Natural Science Foundation of China under Grant 12561070, and Yunnan Fundamental Research Project under Grant 202401AT070479.
	
	%\section*{Declarations}

	%\begin{itemize}
	%\item Declaration of competing interest: The authors declare that they have no conflict of interest.
	%\item Data availability: All datasets are publicly available.
	%\end{itemize}

\end{document}